\documentclass[11pt,a4paper]{article}
\usepackage[hyperref]{acl2019}
\usepackage{times}
\usepackage{latexsym}

\usepackage{url}

\usepackage{latexsym}
\usepackage{booktabs}
\usepackage{soul,color}
\usepackage{todonotes} 
\usepackage[normalem]{ulem}
\usepackage{tikz}

\usepackage{amsfonts,amsmath,amsthm}
\theoremstyle{plain}
\newtheorem{theorem}{Theorem}[section]

\newtheorem{proposition}[theorem]{Proposition}

\usepackage{mathtools}
\DeclarePairedDelimiter{\ceil}{\lceil}{\rceil}

\newcommand{\weight}[1] {\mathbf{W}_{#1}}
\newcommand{\bias}[1] {\mathbf{b}_{#1}}
\newcommand{\hidden}[1] {h_{#1}}

\newcommand{\eqpunc}[1]{{\makebox[0pt][l]{\qquad\rm{#1}}}}

\aclfinalcopy 

\title{LSTM Networks Can Perform Dynamic Counting}

\author{Mirac Suzgun\textsuperscript{1} \hfill Sebastian Gehrmann\textsuperscript{1} \hfill Yonatan Belinkov\textsuperscript{12} \hfill Stuart M. Shieber\textsuperscript{1} \\\\ 
  \textsuperscript{1} Harvard John A. Paulson School of Engineering and Applied Sciences \\  
  \textsuperscript{2} MIT Computer Science and Artificial Intelligence Laboratory \\ 
  Cambridge, MA, USA \\ 
  \texttt{\{msuzgun@college,\{gehrmann,belinkov,shieber\}@seas\}.harvard.edu} }

\date{}

\begin{document}
\maketitle
\begin{abstract}
In this paper, we systematically assess the ability of standard recurrent networks to perform dynamic counting and to encode hierarchical representations. All the neural models in our experiments are designed to be small-sized networks both to prevent them from memorizing the training sets and to visualize and interpret their behaviour at test time. Our results demonstrate that the Long Short-Term Memory (LSTM) networks can learn to recognize the well-balanced parenthesis language (Dyck-$1$) and the shuffles of multiple Dyck-$1$ languages, each defined over different parenthesis-pairs, by emulating simple real-time $k$-counter machines. To the best of our knowledge, this work is the first study to introduce the shuffle languages to analyze the computational power of neural networks. We also show that a single-layer LSTM with only one hidden unit is practically sufficient for recognizing the Dyck-$1$ language. However, none of our recurrent networks was able to yield a good performance on the Dyck-$2$ language learning task, which requires a model to have a stack-like mechanism for recognition.
\end{abstract}

\section{Introduction}
Recurrent Neural Networks (RNNs) are known to capture long-distance and complex dependencies within sequential data. In recent years, RNN-based architectures have emerged as a powerful and effective architecture choice for language modeling \citep{mikolov2010recurrent}. When equipped with infinite precision and rational state weights, RNN models are known to be theoretically Turing-complete \citep{siegelmann1995computational}. However, there still remain some fundamental questions regarding the practical computational expressivity of RNNs with finite precision. 

\citet{weiss2018practical} have recently demonstrated that Long Short-Term Memory (LSTM) models \citep{hochreiter1997long}, a popular variant of RNNs, can, theoretically, emulate a simple real-time $k$-counter machine, which can be described as a finite state controller with $k$ separate counters, each containing integer values and capable of manipulating their content by adding $\pm 1$ or $0$ at each time step \citep{fischer1968counter}. The authors further tested their theoretical result by training the LSTM networks to learn $a^n b^n$ and $a^n b^n c^n$. Their examination of the cell state dynamics of the models exhibited the existence of simple counting mechanisms in the cell states. Nonetheless, these two formal languages can be captured by a particularly simple form of automaton, a deterministic one-turn two-counter automaton \citep{ginsburg1966finite}. Hence, there is still an open question of whether the LSTMs can empirically learn to emulate more general finite-state automata equipped with multiple counters capable of performing an arbitrary number of turns. 

In the present paper, we answer this question in the affirmative. We assess the empirical performance of three types of recurrent networks---Elman-RNNs (or RNNs, in short), LSTMs, and Gated Recurrent Units (GRUs)---to perform \textit{dynamic counting} by training them to learn the Dyck-$1$ language. Our results demonstrate that the LSTMs with only a single hidden unit perform with perfect accuracy on the Dyck-$1$ learning task, and successfully generalize far beyond the training set. Furthermore, we show that the LSTMs can learn the 
shuffles of multiple Dyck-$1$ languages, defined over disjoint parenthesis-pairs, which require the emulation of multiple-counter arbitrary-turn machines. Our results corroborate the theoretical findings of \citet{weiss2018practical}, while extending their empirical observations. On the other hand, when trained to learn the Dyck-$2$ language, which is a strictly context-free language, all our recurrent models failed to learn the language. 

\section{Preliminaries}
\label{sec:preliminaries}
We start by defining several subclasses of deterministic pushdown automata (DPA).  
Following \citet{valiant1975deterministic}, we define a deterministic one-counter automaton (DCA$_1$) to be a DPA with a stack alphabet consisting of only one symbol. Traditionally, this construction allows $\epsilon$-moves (that is, executing actions on the stack without the observance of any inputs), but we restrict our attention to simple DCA$_1$s without $\epsilon$-moves in the rest of this paper. Similarly, we call a DPA that contains $k$ separate stacks, with each stack using only one stack symbol, a deterministic $k$-counter automaton (DCA$_{k}$).\footnote{\citet{weiss2018practical} call such a construction a $k$-counter machine. Previous papers provide a detailed investigation of the complexity of counting machines \citep{minsky1967computation, fischer1968counter, valiant1975deterministic,valiant1975regularity}.}

One can impose a further restriction on the direction of stack movement of a DPA. This notion leads to the definition of a deterministic $n$-turn pushdown automaton (or $n$-turn DPA, in short) and is well-studied by \citet{ginsburg1966finite} and \citet{valiant1974equivalence}: A DPA is said to be an $n$-turn DPA if the total number of direction changes in the stack movement of the DPA is at most $n$ for each stack. Note that a one-turn DCA$_1$ can recognize $a^n b^n$ \citep{valiant1973decision}, whereas a one-turn DCA$_2$ can recognize $a^n b^n c^n$.  We say that a DCA$_k$ with no limit on the number of turns can perform \textit{dynamic counting}. 

\section{Related Work}
Formal languages have long been used to demonstrate the computational power of neural networks. Early studies \citep{steijvers1996recurrent, tonkes1997learning, rodriguez1998recurrent, boden1999learning, boden2000context, rodriguez2001simple} employed Elman-style RNNs \citep{elman1990finding} to recognize simple context-free and context-sensitive languages, such as $a^n b^n$, $a^n b^n c^n$, and $a^n b^n c b^m a^m$. Most of these architectures, however, suffered from the vanishing gradient problem \citep{hochreiter1998vanishing} and could not generalize far beyond their training sets. 

Using LSTMs \citep{hochreiter1997long}, \citet{gers2001lstm} showed that their models could learn two strictly context-free languages and one strictly context-sensitive language by effectively using their gating mechanisms. In contrast, \citet{das1992learning} proposed an RNN model with an external stack memory, named Recurrent Neural Network Pushdown Automaton (NNPDA), to learn basic  context-free grammars. 

More recently, \citet{joulin2015inferring}
introduced simple RNN models equipped with differentiable stack modules, called Stack-RNN, to infer algorithmic patterns, and showed that their model could successfully learn various formal languages, in particular $a^n b^n$, $a^n b^n c^n$, $a^n b^n c^n d^n$, $a^n b^{2n}$. Inspired by the early model design of NNPDAs, \citet{grefenstette2015learning} also proposed memory-augmented recurrent networks (Neural Stacks, Queues, and DeQues), which are RNNs equipped with unbounded differentiable memory modules, to perform sequence-to-sequence transduction tasks that require specific data structures. 

\citet{deleu2016learning} investigated the ability of Neural Turing Machines (NTMs; \citet{graves2014neural}) to capture long-distance dependencies in the Dyck-$1$ language. Their empirical findings demonstrated that an NTM can recognize this language by emulating a DPA. 
Similarly, \citet{sennhauser2018evaluating}, \citet{bernardy2018can}, and \citet{hao2018context} conducted experiments on the Dyck languages to explore whether recurrent networks can learn nested structures. These studies assessed the performance of their recurrent models to predict the next possible parenthesis, assuming that it is a closing parenthesis.\footnote{Their approach is slightly different than ours in the sense that we always try to predict the set of all the possible opening and closing parentheses at each time step.} In fact, \citet{bernardy2018can} used a purpose-designed architecture, called RUSS, which contains recurrent units with stack-like states, to perform the closing-parenthesis-completion task. Though the RUSS model had no trouble generalizing to longer and deeper sequences, as the author mentions, the specificity of the architecture disqualifies it as a practical model choice for natural language modeling tasks. Additionally, \citet{skachkova2018closing} trained recurrent networks to predict the last appropriate closing parenthesis, given a Dyck-$2$ sequence without its last symbol. They showed that their GRU and LSTM models performed with almost full accuracy on this parenthesis-completion task, but their task does not illustrate that these RNN models can recognize the Dyck language.

Most recently, \citet{weiss2018practical} and \citet{suzgun2019evaluating} showed that the LSTM networks can develop natural counting mechanisms to recognize simple context-free and context-sensitive languages, particularly $a^n b^n$, $a^n b^n c^n$, $a^n b^n c^n d^n$. Their examination of the cell states of the LSTMs revealed that the models learned to emulate simple one- and two-turn counters to recognize these formal languages, but the authors did not conduct any experiments on tasks that require counters to perform arbitrary number of turns.

\section{Models}
All the models in this paper are recurrent neural architectures and known to capture long-distance relationships in sequential data. We would like to compare and contrast their ability to perform dynamic counting to recognize simple counting languages. We further investigate whether they can learn the Dyck-$2$ language by emulating a DPA.

A simple RNN architecture
\citep{elman1990finding} is a recurrent model that takes an input $x_t$ and a previous hidden state representation $h_{t-1}$ to produce the next hidden state representation $h_t$, that is:
\begin{align*} 
h_t &= f (\weight{ih} x_t + \bias{ih} + \weight{hh} \hidden{t-1} + \bias{hh}) \\
y_t &= \sigma (\weight{y} \hidden{t})
\end{align*}
where $x_t \in \mathbb{R}^D$ is the input, $\hidden{t} \in \mathbb{R}^H$ the hidden state, $y_t \in \mathbb{R}^D$ the output at time $t$, $\weight{y} \in \mathbb{R}^{D \times H}$ the linear output layer, $f$ an activation function\footnote{In our experiments, we used the \text{tanh} function.} and $\sigma$ an elementwise logistic sigmoid function. 

In theory, it is known that RNNs with infinite precision and rational state weights are computationally universal models \citep{siegelmann1995computational}. However, in practice, the exact computational power of RNNs with finite precision is still unknown. Empirically, RNNs suffer from the vanishing or exploding gradient problem, as the length of the input sequences grow \citep{hochreiter1998vanishing}. To address this issue, different neural architectures have been proposed over the years. Here, we focus on two popular RNN variants with similar gating mechanism, namely LSTMs and GRUs.

The LSTM model was introduced by \citet{hochreiter1997long} to capture  long-distance dependencies more accurately than simple RNNs. It contains additional gating components to facilitate the flow of gradients during back-propagation. 

The GRU model was proposed by \citet{cho2014learning} as an alternative to LSTM. GRUs are similar to LSTMs in their design, but do not contain an additional memory unit. 

\section{Experimental Setup}
\vspace{-0.6em} 
To evaluate the capability of RNN-based architectures to perform dynamic counting and to encode hierarchical representations, we conducted experiments on four different synthetic sequence prediction tasks. Each task was designed to highlight some particular feature of recurrent networks. All the tasks were formulated as supervised learning problems with discrete $k$-hot targets and mean-squared-error loss under the sequence prediction framework, defined next. We repeated each experiment ten times but used the same random seed across each run for each of the tasks to ensure comparability of RNN, GRU, and LSTM models.

\vspace{-0.2em} 
\subsection{The Sequence Prediction Task}
\vspace{-0.2em} 
Following \citet{gers2001lstm}, we trained the models as follows: Given a sequence in the language, we presented one character at each time step to the network and trained the network to predict the set of next possible characters in the language, based on the current input character and the prior hidden state. We used a one-hot representation to encode the inputs and a $k$-hot representation to encode the outputs.
In all the experiments, the objective was to minimize the mean-squared error of the sequence predictions. We used an output threshold criterion of $0.5$ for the sigmoid layer to indicate which characters were predicted by the model.
Finally, we turned this sequence prediction task into a sequence classification task by \textit{accepting} a sequence if the model predicted \emph{all} of its output values correctly and \textit{rejecting} it otherwise. 

\vspace{-0.3em}
\subsection{Training Details}
\vspace{-0.2em}

Given the nature of the four languages that we will describe shortly, if recurrent models can learn them, then they should be able to do so with reasonably few hidden units and without the employment of any embedding layer or the dropout operation.\footnote{Some previous studies used embeddings \citep{sennhauser2018evaluating, bernardy2018can} and the dropout operation \citep{bernardy2018can} in their experiments.} To that end, all the recurrent models used in our experiments were single-layer networks containing less than $10$ hidden units. The number of hidden units that the networks contained for the Dyck-$1$, Dyck-$2$, Shuffle-$2$, and Shuffle-$6$ experiments were $3$, $4$, $4$, and $8$, respectively. (In Section~\ref{sec:discussion}, we describe further experiments with as few as a single hidden unit.) In all our experiments, we used the Adam optimizer \citep{kingma2014adam} with hyperparameter $\alpha = 0.001$.

\begin{table*}[t]
\centering
\begin{tabular}{l | cccccccc | cccccccccccccc}
\toprule
 \bf Sample & \multicolumn{8}{c|}{$( \ ( \ ) \ ) \ ( \ ) \ ( \ )$} & \multicolumn{14}{c}{$( \  [ \ ( \ ) \ ] \  ) \ [ \ ( \ [ \ ]  \ ) \ ] \ [ \ ]$} \\ 
 \midrule
 \bf  Input & $($&$($&$)$&$)$&$($&$)$ &$($&$)$ & $($ & $[$ & $($ & $)$ & $]$ & $)$ & $[$ & $($ & $[$ & $]$ & $)$ & $]$ & $[$ & $]$ \\ 
 \bf Output &  $1$ & $1$ & $1$ & $0$ & $1$ & $0$  & $1$ & $0$ & 
 $1$ & $2$ & $1$ & $2$ & $1$ & $0$ & $2$ & $1$ & $2$ & $1$ & $2$ & $0$ & $2$ & $0$\\
\bottomrule
\end{tabular}
\caption{Example input-output pairs for the Dyck-1 (left) and Dyck-2 (right) languages.}
\label{tab:examples_dyck}
\end{table*}

\begin{table*}[!ht]
\centering
\begin{tabular}{l | cccccccccccc | cccccccccc}
\toprule
 \bf Sample & \multicolumn{12}{c|}{$( \  [ \ ( \ ) \ ) \  ] \ [ \ ( \ [ \ ]  \ ] \ )$} & \multicolumn{10}{c}{$[ \ \{ \ ( \ ) \ \} \ ] \ \langle \ \lceil \ \rceil \ \rangle $}
 \\ \midrule
 \bf  Input & $($ & $[$ & $($ & $)$ & $)$ & $]$ & $[$ & $($ & $[$ & $]$ & $]$ & $)$ 
 & $[$ & $\{$ & $($ & $)$ & $\}$ & $]$ & $\langle$ & $\lceil$ & $\rceil$ & $\rangle$\\ 
 \bf Output & $1$ & $3$ & $3$ & $3$ & $2$ & $0$ & $2$ & $3$ & $3$ & $3$ & $1$ & $0$ 
 & $2$ & $6$ & $7$ & $6$ & $2$ & $0$ & $8$ & $24$ & $8$ & $0$\\
\bottomrule
\end{tabular}
\caption{Example input-output pairs for the Shuffle-2 (left) and Shuffle-6 (right) languages.}
\label{tab:examples_shuffle}
\end{table*}

\section{Experiments}
The first language, Dyck-$1$ (or $D_1$), consists of well-balanced sequences of opening and closing parentheses. Recall that a neural network need not be equipped with a stack-like mechanism to recognize the Dyck-$1$ language under the sequence prediction paradigm; a single counter DCA$_1$ is sufficient. However, dynamic counting is required to capture the language. 

The next two languages are the shuffles of two and six Dyck-$1$ languages, each defined over disjoint parentheses; we refer to these two languages as \textit{Shuffle}-$2$ and \textit{Shuffle}-$6$, respectively. These two tasks are formulated to investigate  whether recurrent networks can emulate deterministic $k$-counter automata by performing dynamic counting, separately counting the number of opening and closing parentheses for each of the distinct parenthesis-pairs and predicting the closing parentheses for the pairs for which the counters are non-zero, in addition to the opening parentheses.  In contrast, the final language, Dyck-$2$, is a context-free language which cannot be captured by a simple counting mechanism; a model capable of recognizing the \mbox{Dyck-$2$} language must contain a stack-like component \citep{sennhauser2018evaluating}. 

Tables~\ref{tab:examples_dyck} and~\ref{tab:examples_shuffle} provide example input-output pairs for the four languages under the sequence-prediction task. For purposes of presentation only, we use a simple binary encoding of the output sets to concisely represent the output. In all of the languages we investigate in this paper, the open parentheses are always allowable as next symbol; we assign the set of open parentheses the number $0$. Each closing parenthesis is assigned a different power of 2: $)$ is assigned to $1$, $]$ to $2$, $\}$ to $4$, $\rangle$ to $8$, $\rceil$ to $16$, and $\rfloor$ to $32$. (These latter closing parentheses are needed for the Shuffle-6 language below.) The set of predicted symbols is then the sum of the associated numbers. For instance, an output 3 represents the prediction of any of the open parentheses as well as $)$ and $]$.

We note that even though an input sequence might appear in two different languages, it might have different target representations. This observation is important especially when making comparisons between the Dyck-$2$ and the Shuffle-$2$ languages. For instance, the output sequence for $( \ [ \ ] \ )$ in the Dyck-$2$ language is $1\, 2\, 1\, 0$, whereas the output sequence for $( \ [ \ ] \ )$ in the Shuffle-$2$ language is $1\, 3\, 1\, 0$.

\vspace{-0.3em}
\subsection{The Dyck-1 Language}
\vspace{-0.2em}
The Dyck-$1$ language, or the well-balanced parenthesis language, arises naturally in enumerative combinatorics, statistics, and formal language theory. A sequence in the Dyck-$1$ language needs to contain an equal number of opening and closing parentheses, with the constraint that at each time step the total number of opening parentheses must be greater than or equal to the total number of closing parentheses so far. In other words, for a given sequence $s = s_1 \cdots s_{2n}$ of length $2n$ in the Dyck-$1$ language over the alphabet $\Sigma = \{(, )\}$, if we have a function $f$ that assigns value $+1$ to `$($' and value $-1$ to `$)$', then we know that it is always true that $\sum_{i = 1}^{j} f(s_i) \geq 0$ with strict equality when $j = 2n$, for all $j \in [1,\ldots, 2n]$. Therefore, a model with a single unbounded counter can recognize this language. 

\vspace{-.2em}
A probabilistic context-free grammar for the Dyck-$1$ language can be written as follows:
\begin{align*}
S \rightarrow \begin{cases} 
( S ) & \text{with probability } p \\
SS & \text{with probability } q \\ 
\varepsilon & \text{with probability } 1 - (p+q) 
\end{cases}
\end{align*}
\noindent where $0 < p, q < 1$ and $(p+q) < 1$.

\begin{table*}[!ht]
\centering
\begin{tabular}{c | c | ccc | ccc | ccc}
\toprule
& & \multicolumn{3}{c|}{\bf{Training Set}} & \multicolumn{3}{c|}{\bf{Short Test Set}} & \multicolumn{3}{c}{\bf{Long Test Set}} \\
\bf{Task} & \bf{Model} & \bf{Min} & \bf{Max} & \bf{Med} & \bf{Min} & \bf{Max} & \bf{Med} & \bf{Min} & \bf{Max} & \bf{Med} \\ \midrule
Dyck-$1$ & RNN & 0.45 & 76.17 & 46.96 & 0.28 & 73.62 & 41.89 & 0.06 & 24.44 & 7.19 \\
Dyck-$1$ & GRU & 99.37 & 100 & 100 & 99.34 & 100 & 100 & 67.68 & 95.58 & 84.38 \\
\bf{Dyck-1} & \bf{LSTM} & \bf{100} & \bf{100} & \bf{100} & \bf{100} & \bf{100} & \bf{100} & \bf{99.98} & \bf{100} & \bf{100} \\ 
\midrule
Shuffle-$2$ & RNN & 33.68 & 87.77 & 58.16 & 29.42 & 86.48 & 55.37 & 0.54 & 25.58 & 2.75 \\
Shuffle-$2$ & GRU & 99.73 & 100 & 99.97 & 99.62 & 99.98 & 99.93 & 83.70 & 95.18 & 93.12 \\
\bf{Shuffle-2} & \bf{LSTM} & \bf{100} & \bf{100} & \bf{100} & \bf{100} & \bf{100} & \bf{100} & \bf{96.84} & \bf{99.98} & \bf{99.35} \\ \midrule
Shuffle-$6$ & RNN & 0.26 & 57.39 & 44.54 & 0.16 & 54.80 & 41.38 & 0 & 0.64 & 0.15 \\
Shuffle-$6$ & GRU & 96.32 & 99.98 & 99.78 & 96.08 & 99.98 & 99.81 & 51.96 & 97.30 & 85.14 \\
\bf{Shuffle-6} & \bf{LSTM} & \bf{99.68} & \bf{100} & \bf{100} & \bf{99.74} & \bf{100} & \bf{100} & \bf{82.92} & \bf{99.72} & \bf{98.14} \\
\midrule
Dyck-$2$ & RNN & 4.37 & 31.67 & 14.74 & 2.96 & 27.46 & 12.21 & 0 & 0.46 & 0.01 \\
Dyck-$2$ & GRU & 7.78 & 53.34 & 28.71 & 5.38 & 49.06 & 25.08 & 0 & 1.56 & 0.05 \\
\bf{Dyck-2} & \bf{LSTM} & \bf{19.76} & \bf{52.13} & \bf{35.82} & \bf{16.58} & \bf{48.24} & \bf{31.29} & \bf{0} & \bf{1.46} & \bf{0.20} \\
\bottomrule
\end{tabular}
\caption{The performances of the RNN, GRU, and LSTM models on four language modeling tasks. Shuffle-$2$ denotes the shuffle of two Dyck-$1$ languages defined over different alphabets, and similarly Shuffle-$6$ denotes the shuffle of six Dyck-$1$ languages defined over different alphabets. Min/Max/Median results were obtained from $10$ different runs of each model with the same random seed across each run. We note that the LSTM models not only outperformed the RNN and GRU models but also often achieved full accuracy on the short test set in all the ``counting'' tasks. Nevertheless, even the LSTMs were not able to yield a good performance on the Dyck-$2$ language modeling task, which requires a stack-like mechanism.}
\label{tab:results}
\end{table*}

Setting $p = \frac{1}{2}$ and $q = \frac{1}{4}$, we generated $10,000$ distinct Dyck sequences, whose lengths were bounded to $[2, 50]$, for the training set. We used two test sets: The ``short'' test set contained $5,000$ distinct Dyck words defined in the same length interval as the training set but distinct from it. The ``long'' test set contained $5,000$ distinct Dyck words defined in the interval $[52, 100]$. Hence, there was no overlap between any of the training and test sets.

\vspace{-0.6em} 
\paragraph{Results:} 
Table~\ref{tab:results} lists the performances of the RNN, GRU, and LSTM models on the Dyck-$1$ language. First, we highlight that all the LSTM models obtained full accuracy on the training set and short test set, whose sequences were bounded to $[2,50]$, in all the ten trials. They were also able to easily generalize to longer and deeper sequences in the long test set: They obtained perfect accuracy in nine out of ten trials and $99.98\%$ accuracy (only $1$ incorrect prediction) in the remaining trial. These results exhibit that the LSTMs can indeed perform unbounded dynamic counting. 

The GRUs yielded an almost similar qualitative performance on the training and first test sets; however, they could not generalize well to longer and deeper sequences. On the other hand, the RNN models performed significantly worse than the first two recurrent models, in terms of their median accuracy rate. We note that similar empirical observations about the performance-level differences between the RNNs, GRUs, and LSTMs for other simple formal languages were also reported by \citet{weiss2018practical} and  \citet{bernardy2018can}.

\vspace{-0.5em} 
\subsection{The Shuffle-$\boldsymbol{k}$ Language}
\vspace{-0.3em} 
Next, we consider two \textit{shuffle} languages, which are both generated by the Dyck-$1$ language. Before describing each task in detail, let us first define the notion of \textit{shuffling} formally. The shuffling operation $||: \Sigma^{*} \times \Sigma^{*} \to \mathcal{P} (\Sigma^{*})$ can be inductively defined as follows:%
\footnote{We abuse notation by allowing a string to stand for its own singleton set.}
\vspace{-0.2em}
\begin{itemize}
\setlength\itemsep{-0.2em}
    \item $u || \varepsilon = \varepsilon || u = \{u\}$
    \item $\alpha u || \beta v = \alpha (u || \beta v) \cup \beta (\alpha u || v)$
\end{itemize}
\noindent for any $\alpha, \beta \in \Sigma$ and $u, v \in \Sigma^{*}$. For instance, the shuffling of $ab$ and $cd$ would be:
\vspace{-.2em}
\begin{align*}
    ab || cd = \{abcd, acbd, acdb, cabd, cadb, cdab\}
\end{align*}
There is a natural extension of the shuffling operation $||$ to languages.  The \textit{shuffle} of two languages $\mathcal{L}_1$ and $\mathcal{L}_2$, denoted $\mathcal{L}_1 || \mathcal{L}_2$, is defined as the set of all the possible interleavings of the elements of $\mathcal{L}_1$ and $\mathcal{L}_2$, respectively, that is:
\begin{align*}
    \mathcal{L}_1 || \mathcal{L}_2 = \bigcup_{\substack{u \in \mathcal{L}_1, \ v \in \mathcal{L}_2}} u || v
\end{align*}
Given a language $\mathcal{L}$, we define its self-shuffling $\mathcal{L}||^2$ to be $\mathcal{L} || \sigma(\mathcal{L})$, where $\sigma$ is an isomorphism on the vocabulary of $\mathcal{L}$ to a disjoint vocabulary. More generally, we define the $k$-self-shuffling \[\mathcal{L}||^k = \left\{
\begin{array}{ll}
   \{\varepsilon\}  & \mbox{if $k=0$}  \\
  \mathcal{L} || \sigma(\mathcal{L}||^{k-1})  & \mbox{otherwise}\eqpunc{.}
\end{array}\right.\]

\paragraph{The Shuffle-2 Language:} 
The first language is $D_1||^2$, the shuffle of $D_1$ and $D_1$, where the first $D_1$ is over the alphabet $\{(,)\}$ and the second over the alphabet $\{[,]\}$.
For instance, the sequence $ ( \ [ \ ) \ ]$ is in $D_1||^2$ but not in $D_2$, whereas $(\ ] \ [ \ )$ is in neither $D_1||^2$ nor $D_2$. Note that $D_2$ is a subset of $D_1||^2$, but not the other way around.\footnote{On the other hand, we highlight once again that the target sequences in $D_1||^2$ are often different from those in $D_2$.}

To generate the training and test corpora, we used a probabilistic context-free grammar for the Dyck-$2$ language, which we will describe shortly, but considered correct target values for the sequences interpreted as per the Shuffle-$2$ language. The training set contained $10,000$ distinct sequences of lengths in $[2, 50]$. As before, the short test set had $5,000$ distinct samples defined over the same length interval but disjoint from the training set, and the long test set had $5,000$ distinct samples, whose lengths were bounded to $[52, 100]$.

\paragraph{The Shuffle-6 Language:} 
The second shuffle language is $D_1 ||^6$, the shuffle of six Dyck-$1$ languages, each defined over different parenthesis-pairs. Concretely, we used the following pairs: $( \ ), [ \ ], \{ \ \}, \langle \ \rangle, \lceil \ \rceil, \lfloor \ \rfloor$. In theory, an automaton with six separate unbounded-turn counters (DCA$_6$) can recognize this language. Hence, we wanted to explore whether our recurrent networks can learn to emulate a dynamic-counting $6$-counter machine. 

\begin{figure}[!h]
\centering
{\includegraphics[width=0.49\textwidth]{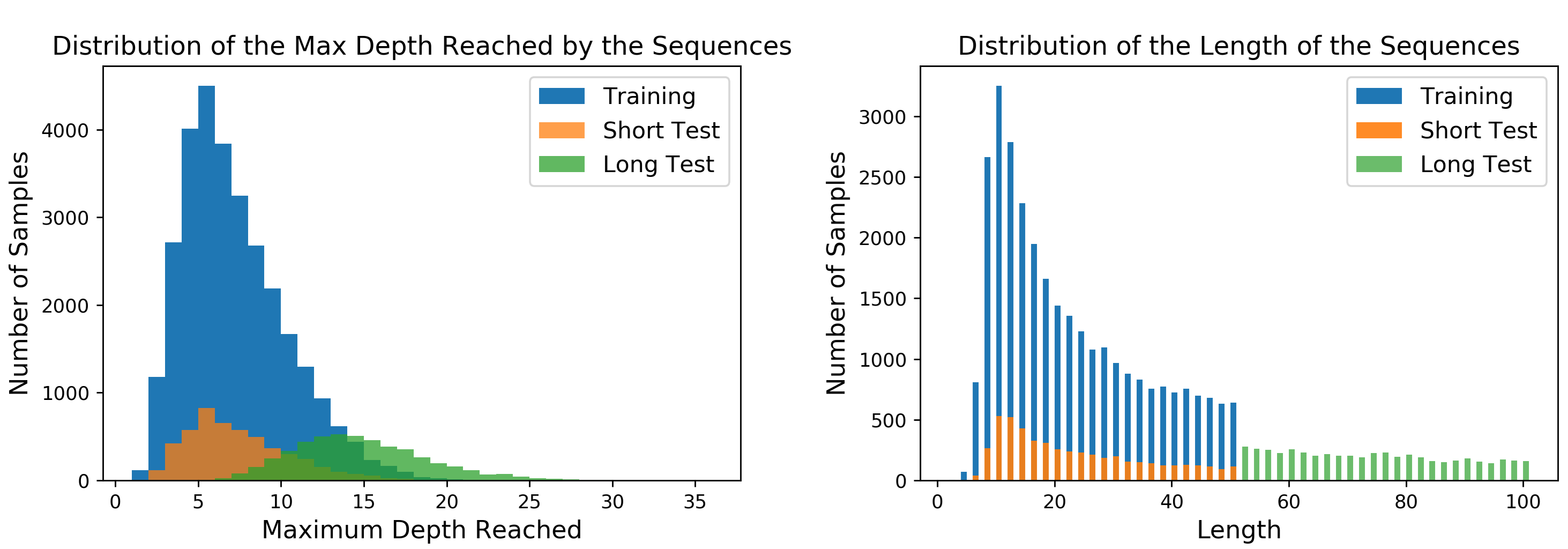}}
\caption{Length and maximum depth distributions of training/test sets for an example Shuffle-6 experiment.}
\label{fig:shuffle6}
\end{figure}

The training and two test corpora were generated in the same style as the previous sequence prediction task; however, we included $30,000$ samples in the training set for this language, due to the complexity of the language. Figure~\ref{fig:shuffle6} shows the length and maximum depth distributions of the training and test sets for one of the Shuffle-$6$ experiments.

\paragraph{Results:} 
As shown shown in Table~\ref{tab:results}, the LSTM models achieved a median accuracy of $100\%$ on the training and short test sets in both of the shuffle language variants. Furthermore, they were able to generalize well to longer and deeper sequences in both shuffle languages, achieving an almost perfect median accuracy score on the long test set. In contrast, the  GRU models performed slightly worse than  the LSTM models on the training and short test sets, but the GRUs did not yield the same performance as the LSTMs on the long test set, obtaining median scores of $93.12\%$ and $85.14\%$ in the Shuffle-$2$ and Shuffle-$6$ languages, respectively. Additionally, the simple RNN models always performed much worse than the GRU and LSTM models and could not even learn the training sets in either of the shuffle languages. These empirical findings show that the LSTM models can successfully emulate a DCA$_k$, a deterministic (real-time) automaton with $k$-counters, each capable of performing an arbitrary number of turns.

\subsection{The Dyck-2 Language}
The generalized Dyck language, $D_n$, represents the core of the notion of context-freeness by virtue of the Characterization Theorem of \citet{chomsky1963algebraic}, which provides a natural way to characterize the CFL class: 
\begin{theorem}
Any language in CFL can be represented as a homomorphic image of the intersection of a Dyck language $D_n$ and a regular language $R$.
\end{theorem}

Furthermore, $D_n$ can be reduced to $D_2$ at the expense of increasing the depth and length of the original sequences in the former language.

\begin{proposition}
$D_n$ is reducible to $D_2$.
\end{proposition}
\begin{proof}\footnote{A similar reduction is also provided by \citet{magniez2014recognizing}.}
Let $D_n$ be the Dyck language with $n$ distinct pairs of parentheses. Let us further suppose that $p = \{p_1, p_2, p_3, \ldots, p_n\}$ are the opening parentheses and that $\bar{p}= \{\overline{p_1}, \overline{p_2}, \overline{p_3}, \ldots, \overline{p_n}\}$ are their corresponding closing parentheses. We set $m = \ceil{\log_2 n}$ and encode each opening and closing parenthesis in $D_n$ with $m$ bits using either $($ and $[$ or $)$ and $]$. Furthermore, we map empty string to empty string.

Given an open parenthesis $p_i$, we first determine the $m$-digit binary representation of the number $i-1$ and then use $($ to encode $0$'s and $[$ to encode $1$'s in this representation. Given a closing parenthesis $\overline{p_i}$, we determine the $m$-digit binary representation of the number $i-1$, write the binary number \emph{in the reverse order}, and then use $)$ to encode $0$'s and $]$ to encode $1$'s. That is, 
\begin{align*}
    \Gamma: p \cup \bar{p} \cup \{\varepsilon\} &\to \mathbb{S}^m \cup \{\varepsilon\} \\
    p_i &\mapsto \mathbf{Enc}_{([} \circ \text{Bin} (i-1) \\
    \overline{p_i} &\mapsto \mathbf{Enc}_{)]} \circ \text{Rev} \circ \text{Bin} (i-1)\\
    \varepsilon &\mapsto \varepsilon
\end{align*}
where $\mathbb{S} = \{(,[, ),]\}$, the parentheses in $D_2$.
We note that $s = s_1 s_2 s_3 \cdots s_k$ is in $D_n$ if and only if $\Gamma (s) = \Gamma(s_1) \Gamma(s_2) \Gamma(s_3) \cdots \Gamma(s_k)$ is in $D_2$, completing the reduction.
\end{proof}

The previous proposition simply shows that we can map an expression in $D_n$ to an expression in $D_2$ at the expense of creating a deeper structure in the latter language by a factor of $m = \ceil{\log_2 n}$. For instance, if an expression $s$ in $D_n$ has a maximum depth of $k$, then the expression generated by the mapping above would have a maximum depth of $k \times m$ in $D_2$. 

Motivated by context-free-language universality \cite{sennhauser2018evaluating}, we therefore experimented with the Dyck-$2$ language defined over two types of parenthesis-pairs, namely $\{(, )\}$ and $\{[,]\}$, as well. The recognition of the Dyck-$2$ language requires a model to possess a stack-like component; real-time primitive counting does not enable us to capture the Dyck-$2$ language. Hence, if an RNN-based architecture learns to recognize this language, we can conclude that RNNs with finite precision can actually learn complex deeply nested representations. 

A probabilistic context-free grammar for the Dyck-$2$ language can be written as follows:
\begin{align*}
S \rightarrow \begin{cases} 
( S ) & \text{with probability } \frac{p}{2} \\
[ S ] & \text{with probability } \frac{p}{2} \\
SS & \text{with probability } q \\ 
\varepsilon & \text{with probability } 1 - (p+q) 
\end{cases}
\end{align*}
\noindent where $0 < p, q < 1$ and $(p+q) < 1$.

Setting $p = \frac{1}{2}$ and $q = \frac{1}{4}$, we generated $10,000$ distinct sequences, whose lengths were bounded to $[2, 50]$, for the training set. Again, we generated $5,000$ other distinct Dyck-$2$ sequences of lengths defined in the interval $[2, 50]$ for the first test set and $5,000$ distinct sequences of lengths defined in the interval $[52, 100]$ for the second test set. As in the previous case, there was no overlap between the training and test sets. 

\vspace{-0.4em} 
\paragraph{Results: }
As shown in Table~\ref{tab:results}, we found that none of our RNNs was able to emulate a DPA to recognize the Dyck-$2$ language, a context-free language that requires a model to contain a stack-like mechanism for recognition. Overall, the LSTM models had the best performances among all the networks, but they still failed to employ a stack-based strategy to learn the Dyck-2 language. Even the best LSTM model could achieve only $48.24\%$ and $1.46\%$ accuracy scores on the short and long test sets, respectively. 

\begin{figure*}[!ht]
\centering
{\includegraphics[width=0.99\textwidth]{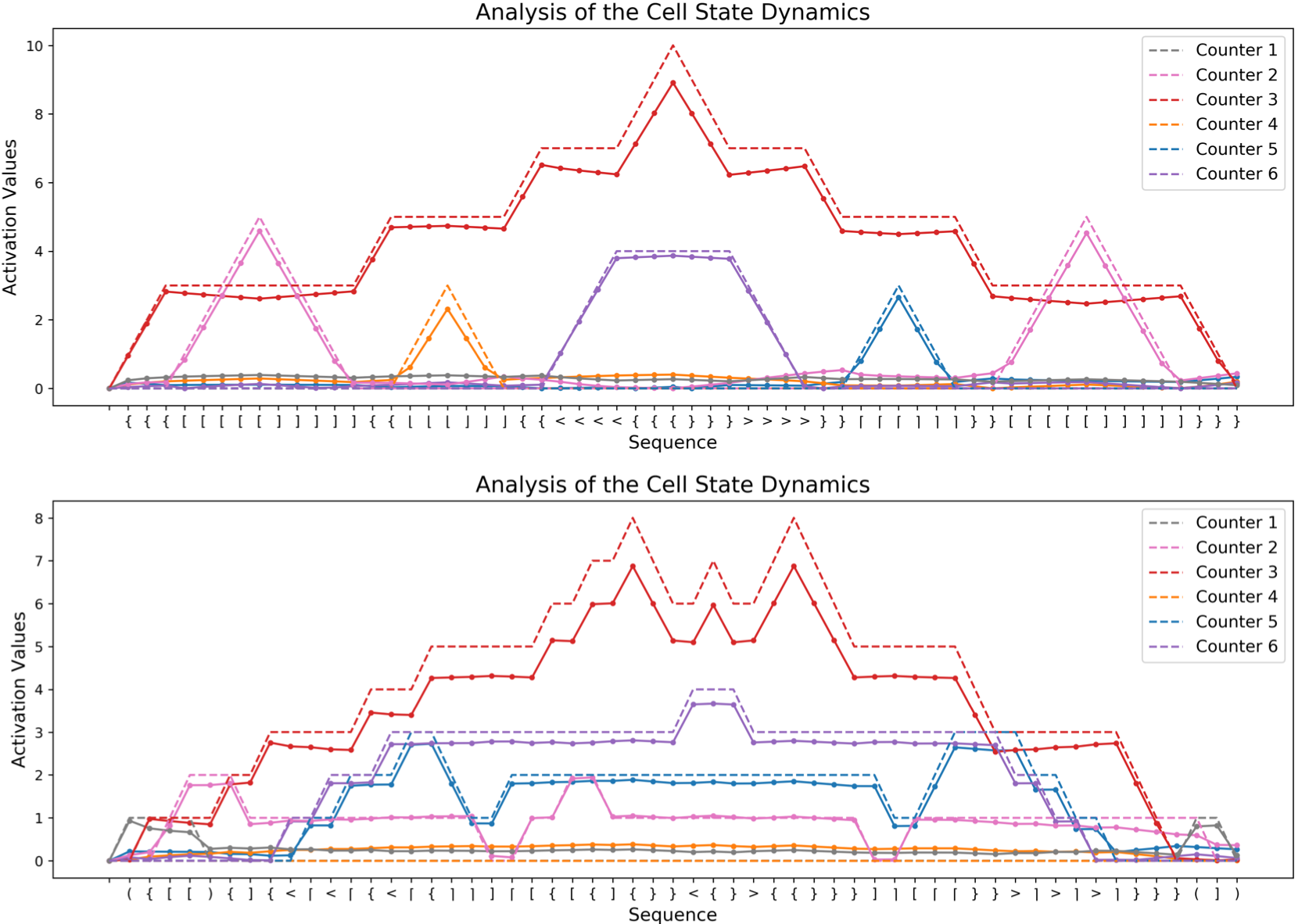}}
\caption{Visualization of the cell state dynamics of one of the LSTM models trained to learn $D_1||^6$, the Shuffle-$6$ language. The solid lines show the values of the cell states of the six out of eight units in the model, whereas the dashed lines depict the current depth of each distinct parenthesis-pair in $D_1||^6$. We highlight the striking parallelism between the solid lines and the dashed-lines. Our visualizations confirm that the LSTM models employ a simple counting mechanism to recognize the Shuffle languages.}
\label{fig:shuffle_graphs}
\end{figure*}

\section{Discussion and Analysis}
\label{sec:discussion}
\subsection{Visualization of Hidden+Cell States} 
Our empirical results on the Dyck-$1$ and Shuffle languages suggest that our LSTM models were performing dynamic counting to recognize these languages. In order to validate our hypothesis, we visualized the hidden and cell states of some of our LSTM models that achieved full accuracy on the test sets. 

Figure~\ref{fig:shuffle_graphs} illustrates that our LSTM is able to recognize the samples in $D_1||^6$ by emulating a DCA$_6$. In fact, the discrete even transitions in the cell state dynamics of the model reveal that six out of eight hidden units in the model are acting like separate counters. In some cases, we further discovered that certain units learned to count the length of the input sequences. Such length counting behaviours are also observed in machine translation \cite{shi2016neural,bau2018identifying,dalvi:2019:AAAI} when the LSTMs are trained on a fixed-length training corpus.\footnote{The visualizations for the Dyck-$1$ and Shuffle-$2$ languages were qualitatively similar.}

\begin{figure*}[!h]
\centering
{\includegraphics[width=0.99\textwidth]{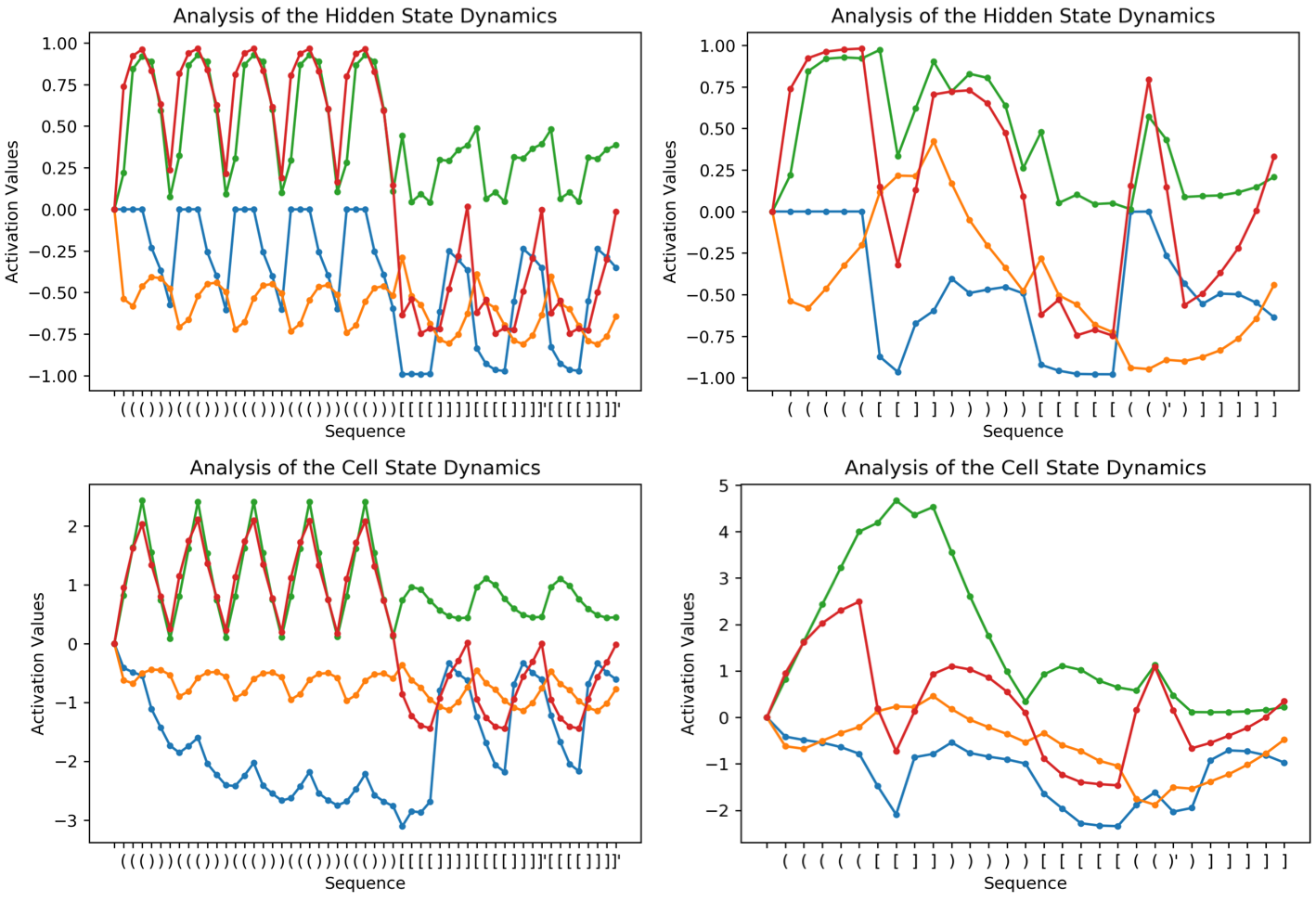}}
\caption{Visualization of the hidden and cell state dynamics of one of the LSTMs trained to learn the Dyck-$2$ language. The time steps at which the model made incorrect predictions are marked with an apostrophe in the horizontal axis. The plots on the left provide a demonstration of the periodic behaviour of the hidden and cell states of the model for a long sequence. Similarly, the plots on the right provide the complex counting behaviour of the model as it observes a nested sequence. We witnessed similar behaviours in our other models as well.}
\label{fig:dyck_graphs}
\end{figure*}

On the other hand, Figure~\ref{fig:dyck_graphs} provides visualizations of the hidden and cell state dynamics of one of our single-layer LSTM models with four hidden units when the model was presented two sequences in the Dyck-$2$ language. Both sequences have some noticeable patterns and were chosen to explore whether the model behaves differently in repeated (or similar) subsequences. It seems that the LSTM model is trying to employ a complex counting strategy to learn the Dyck-$2$ language but failing to accomplish this task.

\subsection{LSTM with a Single Hidden Unit}
In theory, a DCA$_1$ should be able to easily recognize Dyck-$1$, the well-balanced parenthesis language. Can an LSTM with one hidden unit learn Dyck-$1$?
Our empirical results (Figure~\ref{fig:one_unit}) confirmed that LSTMs can indeed learn this language by effectively using the single hidden unit to count up the total number of left and right parentheses in the sequence. Similarly, we found that an LSTM with only two hidden units can recognize $D_1||^2$.

\begin{figure}[!h]
\centering
{\includegraphics[width=0.48\textwidth]{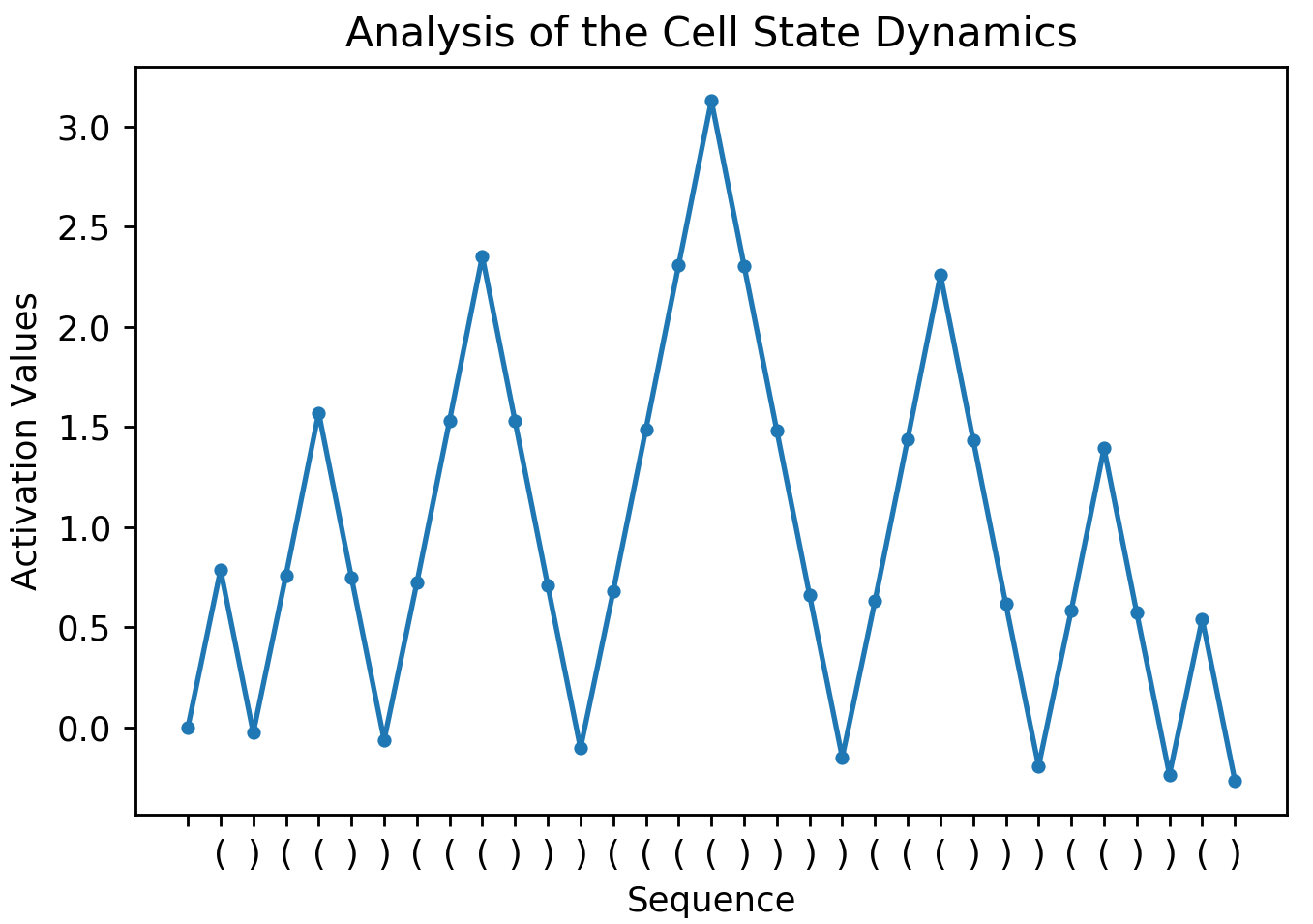}}
\caption{A single-layer LSTM with one hidden unit learns the Dyck-$1$ language by counting up upon the observance of $($ and down upon the observance of $)$.}
\label{fig:one_unit}
\end{figure}

\begin{figure}[t]
\centering
{\includegraphics[width=0.48\textwidth]{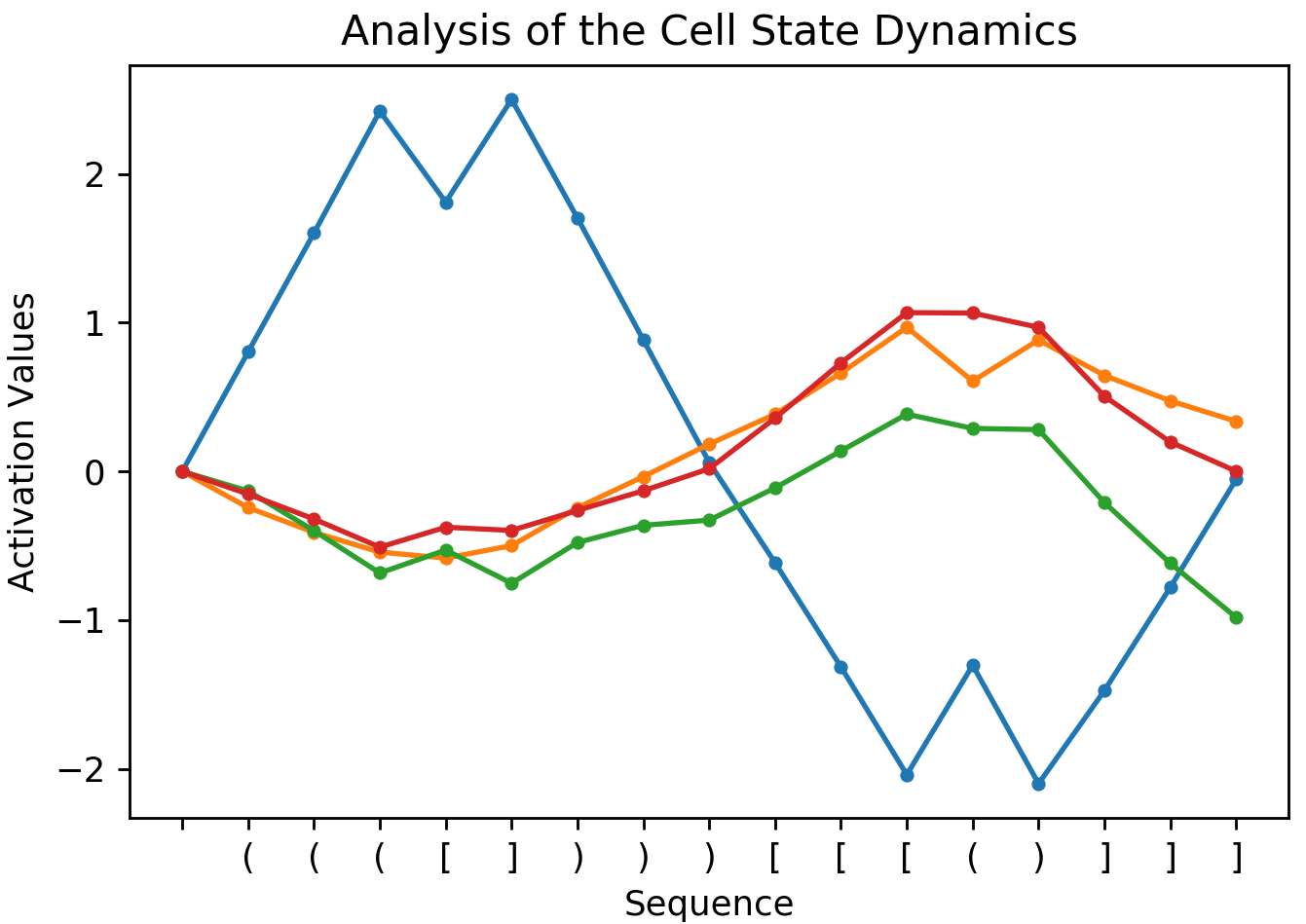}}
\caption{The cell state dynamics of one of the LSTM models trained to predict the last closing parenthesis. Our LSTMs often achieved full accuracy on this task.}
\label{fig:lastclosing}
\end{figure}

\subsection{Predicting the Last Closing Parenthesis}
Following \citet{skachkova2018closing}, we also trained an LSTM model with four hidden units to learn to predict the last closing parenthesis in the Dyck-$2$ language. The model learned the task in a couple of epochs and achieved perfect accuracy on the training and test sets. However, our simple analysis of the cell state dynamics of the LSTM in Figure~\ref{fig:lastclosing} suggests that the model is doing some complex form of counting to perform the desired task, rather than learning the Dyck-$2$ language.

\section{Conclusion}
We investigated the ability of standard recurrent networks to perform dynamic counting and to encode hierarchical representations, by considering three simple counting languages and the Dyck-$2$ language. Our empirical results highlight the overall high-caliber performance of the LSTM models over the simple RNNs and GRUs, and further inflect our understanding of the limitations and strengths of these models.

\section{Acknowledgement}
The first author gratefully acknowledges the support of the Harvard College Research Program (HCRP). The third author was supported by the Harvard Mind, Brain, and Behavior Initiative. The computations in this paper were run on the Odyssey cluster supported by the FAS Division of Science, Research Computing Group at Harvard University.

\newpage

\bibliography{acl2019}
\bibliographystyle{acl_natbib}

\end{document}